\documentclass{article}

\usepackage{arxiv}

\usepackage[utf8]{inputenc} 
\usepackage[T1]{fontenc}    
\usepackage{hyperref}       
\usepackage{url}            
\usepackage{booktabs}       
\usepackage{amsfonts}       
\usepackage{nicefrac}       
\usepackage{microtype}      
\usepackage{lipsum}
\usepackage{graphicx}

\usepackage[utf8]{inputenc}
\usepackage[american]{babel}
\usepackage[babel=true]{csquotes} 
\usepackage{graphicx}
\usepackage{amsfonts,amsmath,mathrsfs,amssymb,bm}
\usepackage{booktabs} 
\usepackage{threeparttable} 
\usepackage{algorithm}
\usepackage[inkscapelatex=false]{svg}
\usepackage{tikz}
\usetikzlibrary{arrows.meta,positioning,calc}
\usetikzlibrary{shapes.geometric, arrows.meta, positioning}
\usepackage{algorithmicx}
\usepackage{algpseudocode}
\usepackage{listings}
\usepackage{enumitem}
\usepackage{chngcntr}
\usepackage{subcaption}
\usepackage{diagbox}
\usepackage{comment}
\usepackage{xcolor}
\usepackage[justification=raggedright, singlelinecheck=false]{caption}
\usepackage{url}

\newtheorem{defi}{Definition}
\newtheorem{thm}{Theorem}
\newtheorem{proof}{Proof}

\DeclareMathOperator{\spline}{spline}
\DeclareMathOperator{\silu}{SiLU}
\algblock{Input}{EndInput}
\algnotext{EndInput}
\algblock{Output}{EndOutput}
\algnotext{EndOutput}

\title{LTBs-KAN: Linear-Time B-splines Kolmogorov-Arnold Networks}

\author{
 Eduardo Said Merin-Martinez \\
  Cinvestav, Unidad Guadalajara\\
  Av. del Bosque 1145, El Bajio\\
  Zapopan, Jalisco, 45017, Mexico \\
  \texttt{eduardo.merin@cinvestav.mx} \\
   \And
 Andres Mendez-Vazquez \\
  Cinvestav, Unidad Guadalajara\\
  Av. del Bosque 1145, El Bajio\\
  Zapopan, Jalisco, 45017, Mexico \\
  \texttt{andres.mendez@cinvestav.mx} \\
  \And
  Eduardo Rodriguez-Tello\\
  Cinvestav, Unidad Tamaulipas\\
  Km. 5.5 Carretera Victoria - Soto La Marina\\
  Victoria, 87139, Tamaulipas, Mexico \\
  \texttt{ertello@cinvestav.mx} \\
}

\date{}

\begin{document}
\maketitle

\begin{abstract}
Kolmogorov-Arnold Networks (KANs) are a recent neural network architecture offering an alternative to Multilayer Perceptrons (MLPs) with improved explainability and expressibility. However, KANs are significantly slower than MLPs due to the recursive nature of B-spline function computations, limiting their application.
This work addresses these issues by proposing a novel base-spline Linear-Time B-splines Kolmogorov-Arnold Network (LTBs-KAN) with linear complexity. Unlike previous methods that rely on the Boor-Mansfield-Cox spline algorithm or other computationally intensive mathematical functions, our approach significantly reduces the computational burden. Additionally, we further reduce model's parameter through product-of-sums matrix factorization in the forward pass without sacrificing performance. Experiments on MNIST, Fashion-MNIST and CIFAR-10 demonstrate that LTBs-KAN achieves good time complexity and parameter reduction, when used as building architectural blocks, compared to other KAN implementations. 
\end{abstract}


\section{Introduction}\label{Introduction}

The recent novel proposal of KANs \cite{liu2024kan} introduces an innovative modification to traditional neural networks by employing learnable functions on the edges of the graph rather than using fixed activation functions at the nodes by Sigmoid Linear Units (SiLU) and spline functions \cite{liu2024kan,Boor1978Splines}. This flexibility enables interpretability and adaptability given the locality of the basic spline. Thus, this locality incorporates adaptability into the new KAN based architecture, improving both accuracy and computational efficiency \cite{liu2024kan}. However, one of the main problems of the KANs lie in its slow processing of the data because of the splines, making it $10$x slower than MLPs, given the same number of parameters \cite{liu2024kan}. Nevertheless, there have been attempts to reduce computational complexity in the KAN architecture \cite{Ta2025PRKAN,Athanasios2024FasterKAN,gottlieb1938polynomials}. 

One of the first attempts to improve KAN performance is the Efficient KAN \cite{efficientkan},
a PyTorch implementation of Kolmogorov-Arnold Network (KAN). Here, the authors propose approximations of the B-spline using a switch activation function that approximates $3^{rd}$-order B-splines used in the original KAN. In addition, it replaces the incompatible $L1$ regularization applied to input samples with the $L1$ penalty on the model weights. Finally, it introduces learnable scaling factors for activation functions and adopting Kaiming uniform initialization \cite{kaiming2015Rectifiers} for both the base weight and spline scaling matrices.

Following this initial attempt, a combination of B-splines and
Radial Basis Functions (BSRBF-KAN) is proposed \cite{Ta2025PRKAN} to fit input vectors during data training.  For this, an RBF network with $N$ centers is used to group samples to a specific basis. Thus, by applying linear transformations, it is possible to align a series of $3^{rd}$-order B-spline bases to Gaussian radial basis. Obtaining a training time two an half that of the classic MLP on the MNIST dataset \cite{LeCunn1998MNIST}, Table \ref{tab:MNIST}.

Furthermore, using the previous ideas, FastKAN \cite{FastKAN2024}
has significantly accelerated model computation to a speed comparable to MLPs, but at the cost of performance \cite{li2024kolmogorovarnold}, Table \ref{tab:MNIST} and Table \ref{tab:Fashion_MNIST}. This is achieved by approximating the $3^{rd}$-order B-spline basis using Radial Basis Functions (RBFs) with Gaussian kernels. Also, layer normalization are used to prevent the inputs shifting away from the domain of the RBFs \cite{li2024kolmogorovarnold}.

Following this, the Gottlieb-KAN \cite{Seydi2024ExploringPolynomial} uses Gottlieb polynomials \cite{gottlieb1938polynomials} which are a special family of polynomials that arises in the study of Bernoulli numbers \cite{gottlieb1938polynomials,Seydi2024ExploringPolynomial}. And although Gottlieb-KAN achieves the best performance across other substitutes for splines, it degrades compared to the MLP architecture, Table \ref{tab:MNIST} and \ref{tab:Fashion_MNIST}.

Finally, we have FasterKAN \cite{Athanasios2024FasterKAN} where FastKAN is used together with Reflective Switch Activation Functions (RSWAF). These functions can accurately approximate a $3^{rd}$-order B-spline basis with linear scaling \cite{Ta2025PRKAN} with a modification for reflectional symmetry. Despite these advances, substantial challenges remain in optimizing Kolmogorov-Arnold Networks (KANs) for practical applications. This gap highlights a significant opportunity to develop algorithms that further enhance the efficiency of KANs. 

Our work addresses these issues by proposing the Linear-Time B-splines Kolmogorov-Arnold Network (LTBs-KAN) model, which efficiently computes the Bernstein–Bézier coefficients of B-spline basis functions. It achieves optimal complexity $O(L d^2(m + n))$  for coincident knots for the forward pass, where $m$ is the order of the spline, $n$ is the number of nodes in the grid, $L$ is the number of layers, and $d$ is the input dimension. For typical configurations using cubic splines and fixed layer counts, the complexity is effectively linear in $n$, i.e., $O(n)$. Further parameter reduction is achieved via product-of-sums matrix factorization in the forward pass, Fig. \ref{fig:architecture}.

Thus, the main contributions of this work are summarized as follows:
\begin{itemize}
    \item A novel parallel linear algorithm LTBs for the B-spline.
    \item The use of Blocks for the implementation of the LTBs-KAN
model.
    \item A matrix factorization for parameter reduction in the model.
    \item  A novel KAN convolutional architecture.
\end{itemize}

The effectiveness of LTBs-KAN is experimentally assessed using MNIST, Fashion-MNIST and CIFAR-10 \cite{LeCunn1998MNIST, Xiao2017FMNIST, Krizhevsky2009CIFAR10} datasets. Thus, LTBs-KAN shows good expressibility, time complexity, and reduction in parameter requirements compared to other KAN architectures. In the case of CIFAR-10, LTBs-KAN units deliver good performance against classic convolution units.
\begin{figure}[tb!]
	\centering
	\includegraphics[width=0.6\textwidth]{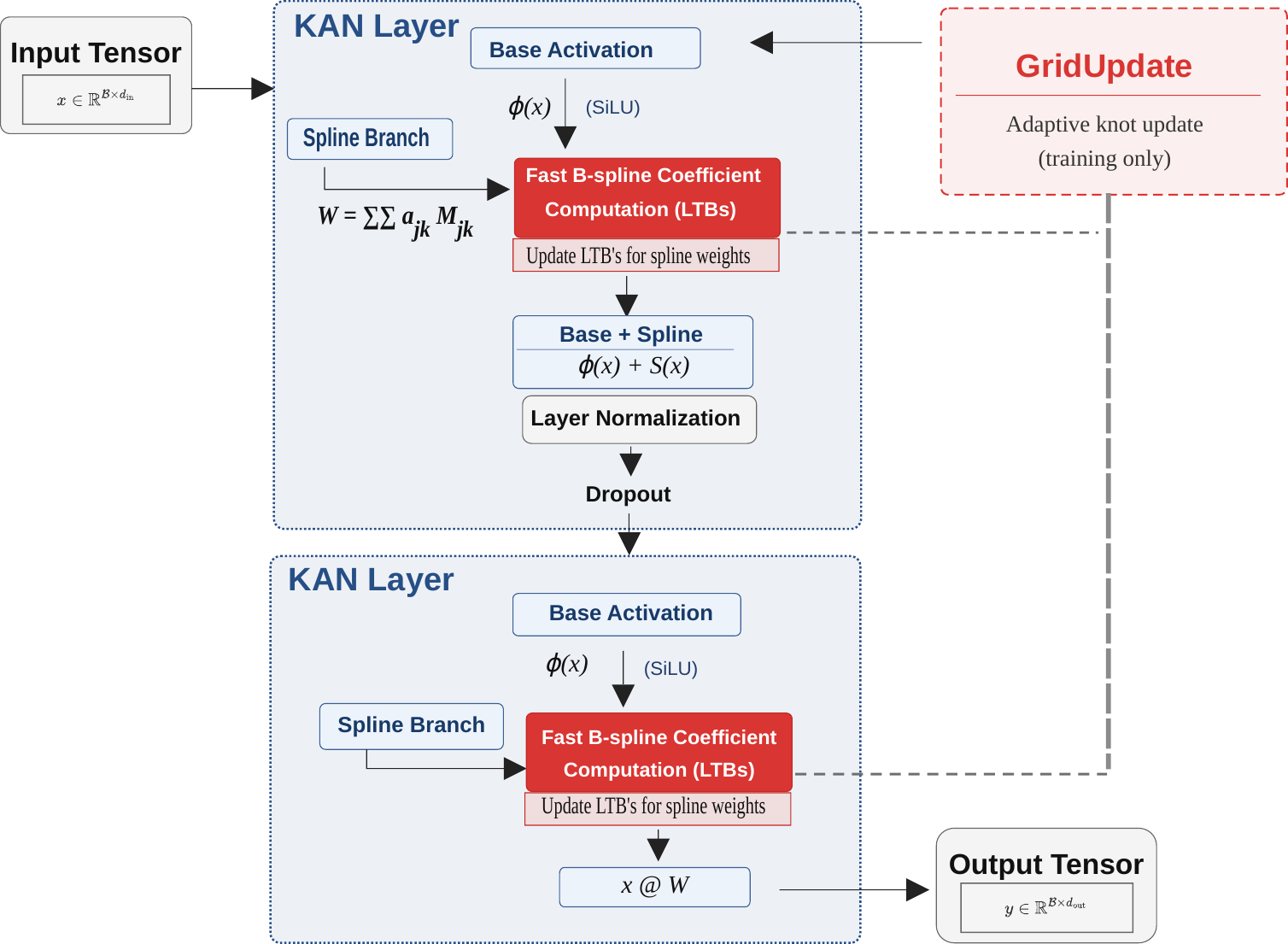}
	\caption{\label{fig:architecture} The proposed LTBs-KAN linear layer using the new LTBs Algorithm, see gridUpdate procedure \ref{algoritmo_updategrid} for details. As part of the general architecture we use dropout and layer normalization for stability and regularization purposes.}
\end{figure}

The remaining sections of this article are organized as follows. Section \ref{sec:antecedents} presents the background concepts used to define the KAN architecture and B-splines. Then, in Section \ref{section:implementation_LBTs_KAN}, our novel LBTs algorithm is introduced to efficiently compute the coefficients of the B-spline basis functions in linear time. Section \ref{secc: KAN-ConvNet} provides the implementation details of LTBs-KAN in CNNs using the proposed LTBs Algorithm. Section \ref{secc:experiments} is dedicated to the experimental evaluation of LTBs-KAN across the MNIST and Fashion-MNIST datasets. This is complemented with experiments using a convolutional architecture for LTBs-KAN on the CIFAR-10 dataset, called KAN-ConvNet. Finally, Section \ref{secc:conclusions}
summarizes the key contributions and future research.

\section{Theoretical background}
\label{sec:antecedents}
The MLP architecture is one of the basic Neural Network structures widely adopted in a range of applications \cite{lecun2015, Haykin1994NeuralNetworks}. The MLP is characterized by a feedforward architecture comprising input, hidden, and output layers, each containing one or more neurons.  Consequently, they serve as standard models in Machine Learning to approximate nonlinear functions \cite{cybenko1989approximation}. 

Deep Learning uses automatic differentiation \cite{griewank2008evaluating} for training deep architectures, resulting in "optimal" solutions that are often not easily interpretable.
This underscores the need for new fundamental units of representation in Deep Learning. The Kolmogorov-Arnold representation theorem \cite{Kolmogorov1957representation,Schmidt-Hieber2021KATRevisited} addresses this need by stating that every continuous multivariate function $f\colon [0,1]^{n}\to \mathbb {R}$, can be represented as a superposition of single variable continuous functions $\phi_{q,p}$ and $\Phi_{q}$:
 {\footnotesize
\begin{equation}\label{formula_KAT}
	f(\textbf{x}) = \sum_{q = 1}^{2n +1} \Phi_{q}\left(\sum_{p=1}^{n}\phi_{q,p}(x_p)\right).
\end{equation}}
This means that the $(2n + 1)(n + 1)$ univariate functions $\phi_{q,p}$ and $\Phi_{q}$ are enough to accurately represent a $n$ 
variable function, where $x_p$ denotes the $p$-th component of the input vector $\boldsymbol{x}\in \mathbb {R}^n$. An example of how to define a possible constructive approach can be found in \cite{Sprecher2002SpaceCurvesKolmogorov} which is time consuming and cumbersome. Thus, neural networks with a more flexible option.

\subsection{KAN architecture}

Given the input $\boldsymbol{x}\in \mathbb{R}^n$, the MLP architecture can be written as a composition of affine transformations $W$ and nonlinearities $\sigma$:
{\small{}
\begin{equation*}
MLP(\boldsymbol{x})  = (W_{L-1}\circ \sigma\circ W_{L-2}\circ\sigma \circ\cdots \circ W_{1} \circ \sigma\circ W_{0})\boldsymbol{x}.    
\end{equation*}
}
One of the main features of this architecture is that it uses fixed activation functions at its nodes, instead KANs use adaptive activation functions for computing their weight connections. For this, each weight parameter is replaced by a SiLU activation function plus a univariate function represented as a spline function \cite{liu2024kan}. Thus, KANs generalize the Kolmogorov-Arnold representation theorem \cite{Kolmogorov1957representation} by using learnable activation functions on edges. In this way, a KANs' layer with $d_\text{in}$-dimensional inputs and $d_\text{out}$-dimensional outputs is defined as: $ 
	\boldsymbol{\Phi} =  \{\phi_{q,p}\}, \quad p = 1,2,\ldots,d_\text{in}, \quad
	q = 1, 2,\ldots,d_\text{out}$, where $ \{\phi_{q,p}\}$ are learnable functions, parameterized as splines.
    The shape of a KAN architecture can be represented by an integer array $[d_0, d_1, \cdots, d_L ]$,
where $d_i$ is the number of nodes in the $i$-th layer of the computational graph. We define the $i$-th neuron in the $l$-th layer by $(l,i)$, and the activation value of the $(l,i)$-neuron by $x_{l,i}$. There are $d_{l}\cdot d_{l+1}$ activation functions between layers $l$ and $l+1$. The activation function that connects $(l,i)$ and $(l+1, j)$ is denoted by $\phi_{l,j,i},  \quad l= 0,\cdots, L-1, \quad i =1, \cdots, d_{l}, \quad j=1, \cdots, d_{l+1}$. The output of a KAN layer is given by:
{\small
\begin{equation}\label{eq_7}
	\boldsymbol{x}_{l+1,j} = \sum_{i= 1}^{d_l}\phi_{l,j,i}(x_{l,i}), \quad j = 1,\ldots, n_{l+1}.
\end{equation}
}

\noindent It can be expressed as: $\boldsymbol{x}_{l+1} = \boldsymbol{\Phi}_l\boldsymbol{x}_{l}$, where $\boldsymbol{\Phi}_l$ is the function matrix corresponding to the $l$-th KAN layer.
The shape of a deep $L$-KAN network is represented by an integer array $[d_0,d_1,\cdots,d_L]$ where $d_l$ denotes the number of neurons in the $l$-th layer. The $l$-th KAN layer, with $d_l$ input dimensions and $d_{l+1}$ output dimensions, transforms
an input vector $x_l \in \mathbb{R}^{d_l}$ to $x_{l+1} \in \mathbb{R}^{d_{l}}$:
{\small
\begin{equation}\label{eq:KAN_layer}
	\boldsymbol{x}_{l+1}=\underbrace{
		\begin{pmatrix}
			\phi_{l,1,1}(\cdot) & \phi_{l,1,1}(\cdot) & \cdots & \phi_{l,1,d_{l}}(\cdot)\\
			\phi_{l,1,2}(\cdot) & \phi_{l,2,2}(\cdot) & \cdots & \phi_{l,2,d_{l}}(\cdot)\\
			\vdots & \vdots &  & \vdots\\
			\phi_{l,d_{l+1},1}(\cdot) & \phi_{l,d_{l+1},2}(\cdot) & \cdots & \phi_{l,d_{l+1},d_{l}}(\cdot)\\
	\end{pmatrix}}_{\boldsymbol{\Phi_l}}\boldsymbol{x}_{l},    
\end{equation}
}
\noindent and the whole network is a composition of $L$ KAN layers: $f(\boldsymbol{x})= (\boldsymbol{\Phi_{L-1}}\circ\cdots\circ\boldsymbol{\Phi_{0}})(\boldsymbol{x})$. We simplify the above equation to make it  analogous to Eq. \eqref{formula_KAT} and we assume output dimension $d_L=1$, finally we define $f(x)\equiv KAN(X)$ \cite{liu2024kan} as:

{\footnotesize
\begin{equation*}
	f(\mathbf{x}) = \sum_{i_{L-1}=1}^{d_{L-1}} \phi_{L-1, i_L, i_{L-1}} \left( 
	\sum_{i_{L-2}=1}^{d_{L-2}} \cdots \left( 
	\sum_{i_1=1}^{d_1} \phi_{1, i_2, i_1} \left( 
	\sum_{i_0=1}^{d_0} \phi_{0, i_1, i_0}(x_{i_0})
	\right)
	\right)
	\right).
\end{equation*}
}

\subsection{Current KANs implementation}\label{subsec1}

The KAN architecture layer described by Eq. \eqref{eq_7} includes an activation function defined as:
{\footnotesize
\begin{equation} \label{eq:phi}
	\phi(x) = w_{b}b(x) + w_{s}\spline(x),
\end{equation}
}
where $b(x)$ is a SiLU function:
{\footnotesize
\begin{equation} \label{eq:basefunction}
	b(x) =  \silu(x) = \frac{x}{1+ \exp(-x)}.
\end{equation}
}

\noindent In addition, $\spline(x)$ is parametrized as a linear combination of B-splines such that $c_j$ are trainable constants.
{\footnotesize
\begin{equation} \label{eq_spline}
	\spline(x) \doteq \sum_{j}c_{j}B_{j}(x).
\end{equation}
}
Each activation function is initialized with $w_s =1$ and $\spline(x) \approx 0$. The value of $w_b$ is based on the classic Xavier initialization \cite{Glorot2010DifficultyTrainDeepNetworks}, which has been successfully used to initialize linear layers in MLPs for training.

\subsection{Function approximation by splines}\label{secc:splines}

In this section, we study B-spline calculation in depth. B-splines are functions that generate smooth interpolation curves, reducing the likelihood of large oscillations typical seen in high-degree polynomials \cite{Piegl1996NurbsBook}. For example, they are used in graphical and numerical methods to solve boundary-value problems for ordinary and partial differential equations. To define a B-spline function as in Eq. \eqref{eq_spline}, we need to introduce several concepts and properties before.

\begin{defi}[\cite{Chudy2023LinearTimeBSpline}, 1.3]\label{def_recu}  The generalized divided difference of a univariate function
	$f:\mathbb{R}^{d}\to \mathbb{R}$ at the knots $x_{i}, x_{i+1},\ldots,x_k $ (which may be coincident), denoted by $[x_{i}, x_{i+1},\ldots,x_k]f$, is defined recursively as:
	\begin{equation*}
	[x_{i}, x_{i+1},\ldots,x_{i+l}]f \\
	\doteq 
	\begin{cases} \frac{[x_{i+1},\ldots,x_{i+l}]f- [x_{i}, x_{i+1},\ldots,x_{i+l-1}]f}{x_{i+l}-x_{i}}, & (x_{i+l} \neq x_{i}) \\
		\frac{f^{(l)}(x_{i})}{l!} & (x_{i+l}= \ldots = x_{i})
	\end{cases}
	\end{equation*}
	in particular, $[x_i]f=\frac{f^{(0)}(x_i)}{0!}= f(x_i)$.
\end{defi}
By using the concept of divided differences, we can define the B-spline basis function recursively. This approach will ultimately enable us to define the B-spline curve accurately and comprehensively.
\begin{defi}[\cite{Chudy2023LinearTimeBSpline}, 1.4]\label{def_bspline} The B-spline basis function $N_{mi}$ of degree $m\in \mathbb{N}$ with knots
	$t_{i}\leq t_{i+1}\leq \cdots \leq t_{i+m+1}$ is defined as
    {\footnotesize
	\begin{equation*}	    
	N_{mi}(u) \doteq  (t_{i+m+1} -t_{i})[t_{i},t_{i+1}, \cdots, t_{i+m+1}](t-u)^{m}_{+},
	\end{equation*}}
	where the generalized divided difference acts on the variable $t$, and
    {\footnotesize
	\begin{equation*}
    	(x-c)^{m}_{+} \doteq 
	\begin{cases} (x-c)^{m} & (x \geq c) \\
		0 & (x < c)
	\end{cases}	    
	\end{equation*}}
	is the truncated power function \cite{Chudy2023LinearTimeBSpline}.
\end{defi}

\noindent In particular, we have that the B-spline function $B_{mi}$ with knots $x_{i}, x_{i+1},\ldots,$ $ x_{m+i+1}$ has support $[t_{i},t_{i}]$, see \cite{Boor1978Splines}. Finally, using these basis functions, the B-spline curve can be defined.
\begin{defi}
	\label{def:bspline_curve}
	A B-spline curve of degree $m$ over the non-empty interval $[a,b]\subset \mathbb{R}$ has knots: $t_{-m}\leq \ldots \leq t_{0}=a\leq t_1 \leq \cdots \leq b =t_n \leq \ldots \leq t_{m+n}$,
	and control points $W_{-m}, W_{-m+1}, \ldots, W_{n-1} \in \mathbb{R}^d$, is defined as
	
    {\footnotesize
    \begin{equation}
		S(t) \doteq \sum_{i = -m}^{n-1}N_{mi}(t)W_{i}, \quad (t\in [a,b]).
	\end{equation}}
	
	\noindent Let us note that  $S([a,b]) \subseteq conv\{ W_{-m},W_{-m+1}, \ldots, W_{n-1} \}$ and $S([t_{i},t_{i+1}]) $ $ \subseteq conv\{ W_{i-m},W_{i-m+1}\ldots , W_{i} \}$ for $0\leq i \leq n -1$, where $conv$ is the convex  hull. (See, e.g., \cite{Piegl1996NurbsBook}, Property 3.5).
\end{defi}

To calculate these basis, a commonly used method is the Boor-Mansfield-Cox formula, see \cite{Boor1972CalculatingBSplines}. This algorithm, stable in polynomial time and numerically, allows the evaluation of spline curves in B-spline form; it is an extension of De Casteljau's algorithm for Bézier curves \cite{Boehm1999CalculatingCasteljauAlgo}.

\begin{thm}[Boor–Mansfield–Cox formula]\label{lem:BMCform}
	The B-spline functions satisfy the following
	Boor–Mansfield–Cox recursion formula:
    {\small
	\begin{equation}\label{eq:12}
		N_{mi}(u) = (u-t_i)\frac{N_{m-1,i}(u)}{t_{m+i}-t_{i}}+(t_{m+i+1}-u)\frac{N_{m-1,i+1}(u)}{t_{m+1+i}-t_{i+1}}, \qquad  -m \leq i < n,
	\end{equation}}
	Additionally, for $i=0, 1,\dots, n-1$,

    {\footnotesize
	\begin{equation}\label{eq:13}
		N_{0i}(u)=\begin{cases} 1 & u \in [t_i , t_{i+1}), \\
			0 & otherwise.
		\end{cases}
	\end{equation}}

	\noindent The derivative of a B-spline function can be expressed as:

{\footnotesize
	\begin{equation}\label{eq:14}
		N_{mi}^{'}(u) = m\left(\frac{N_{m-1,i}(u)}{t_{m+i}-t_{i}}-\frac{N_{m-1,i+1}(u)}{t_{m+1+i}-t_{i+1}}\right),  \qquad  -m \leq i < n,
	\end{equation}}
\end{thm}
\begin{proof}
	The reader is referred to \cite{Boor1972CalculatingBSplines,Boor1978Splines} for the proof of this theorem.
\end{proof}
It is worth noting that the complexity of Boor–Mansfield–Cox formula (Theorem \ref{lem:BMCform}) is $O(nm^{2} + \mathcal{B}m^{2} + \mathcal{B}m^2d_{\text{in}}d_{\text{out}})$, where $m$ denotes the spline order, $n$ represents the grid size, $\mathcal{B}$ is the batch size, $d_\text{in}$ is the number of neurons, and $d_{\text{out}}$ is the number of output features. 
This is clearly too high for the complex applications being used in Deep Learning for inference and training which is the main objective of using the more expressive KAN networks.

\subsection{Decreasing float computations of the b-spline coefficients}
\label{secc:newmethod}

Recently, a novel differential–recurrence relation for B-spline functions of identical degree was established by Chudy et al. \cite{Chudy2023LinearTimeBSpline}. Based on this relation, a recursive procedure is developed to determine the coefficients of B-spline functions of degree $m$ in the Bernstein–Bézier representation with a lower number of divisions.
In \cite{Chudy2023LinearTimeBSpline}, a novel quadratic time calculation of the Bernstein-Bézier basis coefficients is presented, which allows to reduce the number of float division operations with respect to the Boor–Mansfield–Cox formula (Theorem \ref{lem:BMCform}). For this, a new series of coefficients $ b^{(i,j)}_k$ with $m$ denoting the spline degree, $ 0 \le k \le m $, $j=0,...,n-1$ and $i = j - m, j - m + 1, \ldots, j.$ All this, for the new B-spline basis functions over each of the non-trivial knot span $[t_j, t_{j+1}) \subset [t_0, t_n]$.
These new B-spline basis functions $N_{mi}(u)$
admit a Bernstein-Bezier representation of the form (Section 2, Problem 2 in Chudy et al. \cite{Chudy2023LinearTimeBSpline})

{\footnotesize
\begin{equation}
	N_{mi}(u) = \sum_{k=0}^{m}
	b^{(i,j)}_{k} B_k^m\left(
	\frac{u - t_j}{t_{j+1} - t_j}
	\right),
	\qquad
	t_j \le u \le t_{j+1},
	\label{eq:BB-representation}
\end{equation}}
where $ B_k^m(\cdot) $ denotes the Bernstein polynomial of degree $m$  in Boor–Mansfield– Cox formula \ref{lem:BMCform}
on the knot span $ [t_j, t_{j+1}) $.

According to the above problem, in \cite{Chudy2023LinearTimeBSpline} a recursive method is derived to calculate the coefficients of B-spline functions of degree $m$, in the Bernstein-Bézier form. 
For this, let $u\in [t_{j},t_{j+1}]$ and $t\doteq \frac{u-t_{j}}{t_{j+1}-t_j}$, the Bernstein–Bézier coefficients of the B-spline functions have the form:

{\footnotesize
\begin{equation}\label{eq:newcoeffbplsine}
	N_{mi}(u) \doteq  \sum_{k =0}^{m}b_{k}^{(i,j)}B_{k}^{m}(t), \quad (i = j - m, j - m + 1, \cdots , j), \end{equation}}
thus, we can define a point on a B-spline function as:

{\footnotesize
\begin{equation}\label{eq:Newspline}
	S(u) = \sum_{i= j-m}^{j}\bigg(\sum_{
		k=0
	}^{m} b_{k}^{(i,j)}B_{k}^{m}(t)\bigg)W_{i},
\end{equation}}

\noindent assuming that coefficients $b_{0}^{(i,j+1)}$ and $b_{0}^{(i+1,j)}$ $(k = 0, 1, \ldots, m)$ are of the form:

{\footnotesize
\begin{equation}\label{eq:17}
	b_{k}^{(i,j)}=\bigg(\frac{t_{j}-t_{i}}{t_{j+1}-t_i}\bigg)^{m-k}b_{0}^{(i,j+1)} + \sum_{l=0}^{m-k-1}\bigg(\frac{t_{j}-t_{i}}{t_{j+1}-t_i}\bigg)^{l}\frac{v_i}{t_{j+1}-t_i}q_{k+l},
\end{equation}
}
where $q_{l} \doteq (t_{j+1}-t_{m+i+2})b_{l}^{(i+1,j)} + (t_{m+i+2}-t_{j})b_{l+1}^{(i+1,j)}$
and $0\leq j \leq n-2$, $j-m+1 \leq i \leq j-1$. Hence, the coefficients of $N_{m,i}$ for $j = 0, 1, \ldots, n - 1$ and 
$i = j - 1, j - 2, \ldots, j - m + 1$ are fully characterized. In accordance with Chudy et al. \cite{Chudy2023LinearTimeBSpline}, they satisfy the recursive differential relation given in Theorem $3.1$ (Section 3) of the mentioned work.

The complexity of this approach is now $O(nm^{2} + \mathcal{B}m
+ \mathcal{B}md_{\text{in}}d_{\text{out}})$, and although the asymptotic computational complexity of the proposed method may appear comparable to that of the approach based on Boor–Mansfield–Cox formula \ref{lem:BMCform}, a more detailed analysis of the floating-point operation count reveals a substantial difference. Specifically, the proposed method requires only $O(\mathcal{B}m)$ division operations, while the method derived in Boor–Mansfield–Cox formula \ref{lem:BMCform} involves $O(\mathcal{B}m^2)$ divisions. Since division operations are typically more computationally expensive and less numerically stable than additions and multiplications, this reduction is significant in practical implementations. Consequently, the proposed approach eliminates $O(\mathcal{B}m^2)$ division operations, leading to a tangible improvement in computational efficiency.

\section{Parallel computation of b-spline coefficients}
\label{section:implementation_LBTs_KAN}

Now, when implementing the recurrence relation Eq. \eqref{eq:17}, an algorithm is proposed that corresponds to the approach introduced in Theorem 4.4 presented in \cite{Chudy2023LinearTimeBSpline}, that calculates the coefficients of the adjusted Bernstein-Bézier form of the B-spline functions. 
This algorithm  returns a sparse array 
\begin{equation}\label{eq:sparsearray}
	B \doteq B[0..n -1, -m..n - 1, 0..m],
\end{equation}
\noindent with $B[j,i, k] = b_{k}^{(i,j)}$  (Eq. \ref{eq:17}), for $0 \leq j < n$, $-m \leq i < n$ and $0 \leq k \leq m$,  as stated in Section 5 of \cite{Chudy2023LinearTimeBSpline}. Next, let us consider that a usual choice for the boundary knots is to make them coincident with $t_0$ and $t_n$, i.e.,
\begin{equation}
	t_{-m} = t_{-m+1} = \cdots = t_{-1} = t_{0} = a, \quad b = t_{n} = t_{n+1} = \cdots = t_{n+m}.    
\end{equation}
In this case, $S(a) = W_{-m}$ and $S(b) = W_{n-1}$. Please refer to Sections 3 and 4 in \cite{Chudy2023LinearTimeBSpline} for details on the demonstrations. 

In this section, we present the details of our PyTorch-based implementation\footnote{The code and implementation of LTBs-KAN can be found at \url{https://github.com/LalosMerin}} of the LTBs-KAN model. This flexible and high-performance framework is selected for developing the proposed LTBs Algorithm \ref{algoritmo} because it provides tensor computation with GPU acceleration, allowing to significantly accelerate floating number operations by parallel computation. The LTBs-KAN model, as outlined in Algorithm \ref{algoritmo}, implements the core concepts presented in the previous section. Furthermore, it incorporates the recurrence relation derived for the adjusted Bernstein-Bézier coefficients of B-spline functions of the same degree, thereby enhancing the implementation's computational efficiency.

It is important to note that LTBs-KAN can be extended to various knot multiplicities, thereby speeding up the computation of B-spline basis functions, and the evaluation of curves and surfaces. Consequently, this method can be applied to an adjusted Bernstein–Bezier or power basis to efficiently compute the coefficients of polynomial functions $N_{mi}$ within each knot span.

\subsection{The block concepts}
\label{secc:blockconcepts}

In this part, we will show the practical application of the algorithm described in Section \ref{secc:newmethod} (Chudy et al.) in the implementation of our proposed LTBs Algorithm in parallel computing. The new method consistently performs faster because the parallel implementation (see Section \ref{linearKAN:complexity} for this) than evaluating B-spline functions using the recurrence relation Eq. \eqref{eq:17}. 
Next, we consider a vector knot $t = (t_0,\dots,t_{n+2m}) \in \mathbb{R}^{n+2m+1}$, 
and by applying the method described by Chudy et al. \cite{Chudy2023LinearTimeBSpline}, the following tensor is obtained $\mathrm{C} \in \mathbb{R}^{\,n \times (n+m) \times (m+1)},$
which represent the polynomial coefficients associated with the underlying B-spline basis on the specified grid, see Eqs. \eqref{eq:17} and \eqref{eq:newcoeffbplsine}.  Next, we will introduce the logic design of LTBs Algorithm \ref{algoritmo}, which is divided into the following two blocks concepts:

\paragraph{Block 1 Boundary initialization.}
For each cell index $j = 0,\dots,n-1$, the method computes endpoint coefficients of degree $0$ and $m$ using expressions of the form $(t_{j+1+m} - t_{j+m})^{m-1},$
followed by a sequence of normalizing divisions by knot differences, based on the value of $t$ in Eq. \eqref{eq:newcoeffbplsine} . These steps correspond to the classical recursive construction of endpoint polynomials of B-splines. This is implemented in  lines $6-28$
in  LTBs Algorithm \ref{algoritmo}.
\paragraph{Block 2 Backward recursion.}
For $i$ decreasing from $n-2$ to $n-m-1$, and for all $k = m-1,\dots,0$, the
algorithm updates $\mathrm{C}(n-1,i+m,k)$ as

\begin{equation}\label{eq:backwardblock2}
	\mathrm{C}(n-1,i+m,k)
	=
	c_1\mathrm{C}(n-1,i+m,k+1)
	+
	c_2\mathrm{C}(n-1,i+1+m,k+1),
\end{equation}

\vspace*{-3mm}
\begin{algorithm}
	\caption{LTBs Algorithm: Parallel calculation of the coefficients of the adjusted Bernstein-Bézier form of B-spline functions}
	\label{algoritmo}
    \begin{footnotesize}
	\begin{algorithmic}[1]
		\State \textbf{Input} $n, m \in \mathbb{N}$ \Comment{Number of basis functions and spline degree}
		\State \textbf{Input} $\text{knots} \in \mathbb{R}^{n+m+1}$ \Comment{Knot vector}
		\State \textbf{Output} $\text{C} \in \mathbb{R}^{n \times (n+m) \times (m+1)}$ \Comment{Coefficient tensor updated  in-place}
		
		\State $\epsilon \gets 10^{-8}$ \Comment{Numerical stability constant}
		
		\Statex \Comment{\textbf{Block 1: Forward initialization of diagonal terms}}
		\For{$j = 0$ \textbf{to} $n-1$}
		\If{$j + m + 1 < \text{len}(\text{knots})$}
		\State $\Delta \gets \max(\text{knots}[j+m+1] - \text{knots}[j+m], \epsilon)$
		\State $\text{num} \gets \Delta^{m-1}$
		
		\Statex \Comment{Compute upper diagonal term $\text{C}[j, j+m, m]$}
		\State $\text{val}_m \gets \text{num}$
		\For{$k = 2$ \textbf{to} $m$}
		\If{$j + k + m < \text{len}(\text{knots})$}
		\State $\text{d} \gets \max(\text{knots}[j+k+m] - \text{knots}[j+m], \epsilon)$
		\State $\text{val}_m \gets \text{val}_m / \text{d}$
		\EndIf
		\EndFor
		\State $\text{C}[j, j+m, m] \gets \text{val}_m$
		
		\Statex \Comment{Compute lower diagonal term $\text{coeffs}[j, j, 0]$}
		\State $\text{val}_0 \gets \text{num}$
		\For{$k = 2$ \textbf{to} $m$}
		\State $\text{idx} \gets j + 1 - k + m$
		\If{$0 \leq \text{idx} < \text{len}(\text{knots})$}
		\State $\text{d} \gets \max(\text{knots}[j+1+m] - \text{knots}[\text{idx}], \epsilon)$
		\State $\text{val}_0 \gets \text{val}_0 / \text{d}$
		\EndIf
		\EndFor
		\State $\text{C}[j, j, 0] \gets \text{val}_0$
		\EndIf
		\EndFor
		
		\Statex \Comment{\textbf{Block 2: Backward recurrence for remaining coefficients}}
		\For{$i = n-2$ \textbf{downto} $n-m-1$}
		\If{$(i + m) < \text{len}(\text{knots})$ \textbf{and} $(n + m) < \text{len}(\text{knots})$}
		\State $c_1 \gets \dfrac{\text{knots}[n-1+m] - \text{knots}[i+m]}{\text{knots}[n+m] - \text{knots}[i+m] + \epsilon}$
		\State $c_2 \gets \dfrac{\text{knots}[n+m] - \text{knots}[n-1+m]}{\text{knots}[n+m] - \text{knots}[i+1+m] + \epsilon}$
		
		\For{$k = m-1$ \textbf{downto} $0$}
		\State $\text{C}[n-1, i+m, k] \gets c_1 \cdot \text{C}[n-1, i+m, k+1]$
		\State $\quad + c_2 \cdot \text{C}[n-1, i+1+m, k+1]$
		\EndFor
		\EndIf
		\EndFor
		\State \Return $\text{C}$
	\end{algorithmic}
    \end{footnotesize}
\end{algorithm}
\noindent where
\begin{equation*}
	c_1 = \frac{t_{n-1+m}-t_{i+m}}{t_{n+m}-t_{i+m}}, \quad\text{and}\quad
	c_2 = \frac{t_{n+m}-t_{n-1+m}}{t_{n+m}-t_{i+1+m}},    
\end{equation*}

\noindent this result is a complete tensor representation of the local spline polynomial pieces according to the value of $b_{k}^{(i,j)}$ in Eq. \eqref{eq:17}. This is implemented in lines $29-38$ at LTBs Algorithm \ref{algoritmo}.

Therefore, we are ready to explain the new implementation of this LTBs Algorithm in a new KAN architecture called LTBs-KAN, which improves the performance of this architecture compared to other KAN models.

\begin{algorithm}[tb!]
	\caption{gridUpdate Procedure: recalculation of knots and coefficients }
	\label{algoritmo_updategrid}
    \begin{footnotesize}
	\begin{algorithmic}[1]
		
		\Require Input matrix $X \in \mathbb{R}^{\mathcal{B} \times d_{in}}$
		\Require Margin parameter $\delta$
		\Require Grid size $n$, spline order $m$
		
		\State $x_{min} \gets \min(X)$
		\State $x_{max} \gets \max(X)$
		
		\State $span \gets \max(x_{max}-x_{min}, 10^{-3})$
		
		\State $low \gets x_{min} - \delta \cdot span$
		\State $hi \gets x_{max} + \delta \cdot span$
		
		\State $grid \gets \text{linspace}(low, hi, n + 2m + 1)$
		
		\State Update internal knot vector with $grid$
		
		\State $knots \gets grid$
		
		\State Initialize coefficient tensor
		$\mathrm{C} \in \mathbb{R}^{n \times (n+m) \times (m+1)} $
		
		\State $\text{LTBs}(n,m,knots, \mathrm{C})$
		
		\State $o_{fill} \gets \min(d_\text{out},n)$
		\State $i_{fill} \gets \min(d_\text{in},n+m)$
		\State $d_{fill} \gets \min(D,m+1)$
		
		\If{$o_{fill} > 0$ and $i_{fill} > 0$ and $d_{fill} > 0$}
		\State Copy compatible block
		\[
		W_{spline}[0:o_{fill},0:i_{fill},0:d_{fill}]
		\gets
		\mathrm{C}[0:o_{fill},0:i_{fill},0:d_{fill}]
		\]
		\EndIf
	\end{algorithmic}
    \end{footnotesize}
\end{algorithm}

\subsection{Adaptive knot and grid updates}
\label{LTBs-KAN:AdaptiveKnotGridUpdates}

Now, in this section we will make use of the implementation of the LTBs Algorithm explained in the concepts of the previous Section \ref{secc:blockconcepts}. Thus, we take a detailed look at the grid Update procedure (Algorithm \ref{algoritmo_updategrid}) in LTBs-KAN. Given a batch $X$, the knot vector is adapted to the empirical range of the data by setting the following bounds:
\begin{equation}
	\mathrm{low} = x_{\min} - \delta(x_{\max}-x_{\min}), 
	\quad \text{and} \quad
	\mathrm{hi} = x_{\max} + \delta(x_{\max}-x_{\min}),  
\end{equation}
and then recomputing the grid as a uniform subdivision of $[\mathrm{low},\mathrm{hi}]$. The new spline coefficients are obtained by the recursion Eq. \eqref{eq:17}, and the corresponding sub-block is copied into the tensor $\mathrm{C}\left(o,i,d\right)$.
Only the compatible portion of $\mathrm{C}$ is overwritten, which preserves the asymptotic computational complexity of the layer. Therefore, the method \text{GridUpdate} adapts the spline knot vector to the current range of input data. Let us consider
{\footnotesize
\begin{equation}
	x_{\min} = \min(x), \quad \text{and} \quad x_{\max} = \max(x),
\end{equation}}
and define
{\footnotesize
\begin{equation}
	\mathrm{low} = x_{\min} - \delta(x_{\max}-x_{\min}),\quad \text{and} \quad
	\mathrm{hi} = x_{\max} + \delta(x_{\max}-x_{\min}),
\end{equation}}
where $\delta$ is a marginal parameter, using these equations we ensure a stable and well-defined spline representation, the layer
constructs a new knot vector $\hat{t} = \operatorname{linspace}(\mathrm{low}, \mathrm{hi}, n + 2m + 1)$,  where the function $\operatorname{linspace}$ generates a uniformly spaced sequence of $n + 2m + 1$ knots over the interval $[\mathrm{low}, \mathrm{hi}]$. The additional
$2m$ knots provide the necessary boundary multiplicity to ensure that B-splines of order $m$ are properly defined at the domain boundaries. Once constructed, the vector $\hat{t}$ is internally stored within the layer and serves as the reference knot configuration for all subsequent spline computations. In particular, the evaluation of the basis functions and the recurrence used to compute the adjusted Bernstein–Bézier coefficients (LTBs Algorithm \ref{algoritmo}) both depend on this fixed knot structure. This approach ensures that the spline activation component of the LTBs-KAN layer remains constant during training, leaving only the spline coefficients and the linear parameters to be updated. Subsequently, new spline coefficients are computed using the Algorithm \ref{algoritmo} which are stored in the tensor $\mathrm{C}(o,i,d)$.
As a result, only the compatible sub-block is overwritten, thereby preserving computational complexity and ensuring numerical reliability, see Eq. \eqref{eq:backwardblock2}.

\subsection{Parameters reduction}
\label{secc:parametersmodel}

The computational complexity of a deep neural network is often very costly, therefore, differente deep learning techniques to address these problems have been proposed in the literature. Pruning for example, is used to reduce the number of parameters of a trained Neural Network while preserving the original performance. This is because many parameters contribute little or nothing to the final inference made by the Neural Network. In this section, we shall reduce the parameters of LTBs-KAN model using an alternative method to the traditional affine linear transformation in MLPs, with $\boldsymbol{y}=\boldsymbol{x}W^\top + \boldsymbol{b}$, where $\boldsymbol{x}\in \mathbb{R}^{d_\text{in}} $ is the input vector, $W$ is the weights matrix, and $b$ is the bias. We propose using sum-product decompositions of matrices for the forward step. The main idea behind this is to reduce the full matrix $W\in \mathbb{R}^{d_\text{in} \times\ d_\text{out}}$, where $d_\text{in}$ is the number of input features per data sample and $d_\text{out}$ is the number of output features produced by the transformation, with a parametrized version that has fewer parameters than the number of entries \cite{Wu2019ProdSumNet}. 

Let us denote the batch size as $\mathcal{B}$ and the input and output dimensions as $d_{\mathrm{in}}$ and $d_{\mathrm{out}}$, respectively. Also, let $n$ be the grid size, $m$ the spline order, and fixed values $p, s \in \mathbb{N}$. We build a linear operator $W : \mathbb{R}^{\mathrm{d_{in}}} \to \mathbb{R}^{\mathrm{d_{out}}}$, defined as:

\begin{equation}\label{eq:W_weights}
	W \doteq \sum_{j=1}^{p}
	\sum_{k=1}^{s}
	a_{jk} M_{jk},
\end{equation}
where each $a_{jk} \in \mathbb{R}$ is a learnable scalar parameter and each  
$M_{jk} \in \mathbb{R}^{d_\text{in} \times d_\text{out}}$ is a fixed matrix stored as a buffer, i.e., in tensor form.  Thus, $W\subseteq \mathcal{S}$, where $
	\mathcal{S}= \mathrm{span}\{ M_{jk} : 1 \le j \le p,\; 1 \le k \le s \}.$
Let us note that the operator applied to a batch $X \in \mathbb{R}^{\mathcal{B} \times \mathrm{d_\text{in}}}$ is:
\begin{equation}
	XW =\sum_{j=1}^{p} \sum_{k=1}^{s}
	a_{jk} (X M_{jk}),    
\end{equation}
this shows that the learnable linear map is a weighted sum of the fixed linear transformations $X \mapsto XM_{jk}$. 

Each layer of the LTBs-KAN architecture combines these steps: an optional grid update step using the current batch, a factorized linear operator belonging to a prescribed finite-dimensional subspace of $\mathcal{L}(\mathbb{R}^{\mathrm{in}},\mathbb{R}^{\mathrm{out}})$, a spline interaction term defined by tensor contraction with a coefficient tensor and a normalization step. Thus, LTBs-KAN architecture implements a deep composition of structured linear operators and spline interactions, with adaptive knot geometry (Section \ref{LTBs-KAN:AdaptiveKnotGridUpdates}) driven by the data. 

In the original KAN, the Output of each layer decomposes into a linear component and a nonlinear spline component:
{\footnotesize
\begin{equation}
	\mathrm{Output}
	=
	\underbrace{\mathrm{baseOutput}}_{\text{linear term}}
	+
	\underbrace{\text{splineOutput}}_{\text{nonlinear spline term}}.
\end{equation}}
The  term $\text{baseOutput}$ represents the classic linear transformation that forms the basis of the layer. Therefore, given an input batch $X \in \mathbb{R}^{\mathcal{B} \times d_\text{in}}$, the $\text{baseOutput}$ is defined as
{\footnotesize
\begin{equation*}
\mathrm{baseOutput}(X)\doteq
\phi(X)\, W= \sum_{j,k} a_{jk} (X M_{jk}),    
\end{equation*}}
where $\phi$ is a base activation function (e.g.,\ SiLU) and
$W \in \mathbb{R}^{d_\text{in} \times d_\text{out}}$ is the effective linear weight matrix Eq. \eqref{eq:W_weights}. The tensor of spline coefficients is $N \equiv \text{spline weight} \in \mathbb{R}^{d_\text{out} \times d_\text{in} \times D},$ $D \doteq n + m.$
For each input-output pair $(i,o)$, this model defines a univariate spline
\begin{equation}
	\text{spline}_{o,i}(x_i)
	\doteq \sum_{j=1}^{D} \mathrm{C}_{o,i,j}\, B_j(x_i),
\end{equation}
where $\{B_j\}$ is the B-spline basis associated with the current knot vector.
The corresponding batch operator is the contraction

{\footnotesize
\begin{equation}
	N_{b,o}
	=
	\sum_{i=1}^{d_\text{out}}
	\sum_{j=1}^{D}
	X_{b,i}\, C_{o,i,j},
\end{equation}}
implemented as an Einstein summation, according Eq. \eqref{eq:newcoeffbplsine}.
Thus, the spline weight tensor $N$ encodes the coefficients of the local polynomial
representation which defines the nonlinear functional interaction between inputs and outputs. Finally, combining a structured linear operator with a spline-based functional term the layer output can be expressed as follows: 
\begin{equation}
	\boldsymbol{y} = \mathrm{Layer Normalization}\left(XW +\sum_{i,j} X_{b,i} \mathrm{C}_{o,i,j}
	\right).   
\end{equation}

\subsection{
	Computational complexity of LTBs-KAN}
\label{linearKAN:complexity}

As in Section \ref{secc:parametersmodel}, let us consider an input $\boldsymbol{x}\in \mathbb{R}^{d_\text{in}}$  along with a network layer characterized by an input dimension $d_\text{in}$ and an output dimension $d_\text{out}$. Let $m$ represent the spline order and $n$ denote the grid size of a function 
in a KAN. The number of control points (also the number of basis functions) required is $m+n$. The total number of parameters, including the weight matrix and the bias term, when passing $\boldsymbol{x}$ through a KAN layer is:

\begin{equation*}
	KAN_{params} = \underbrace{d_\text{in}  d_\text{out} (m+n)}_{\text{weight matrix params}} \quad + \quad d _\text{out}.   
\end{equation*}  



\noindent Furthermore, we have an $L$-layer KAN with width $\text{N}$ (which means each layer has $\text{N}$ neurons), there are in total $O(\text{N}^{2}L(m+n))\approx O(n\text{N}^2L)$ parameters, where $m$ is the spline order and $n$ the number of grid nodes. Contrarily, an MLP with depth $L$ and width $\text{N}$ typically requires $O(\text{N}^2L)$ parameters, suggesting that it might be more parameter-efficient than a KAN. However, KANs often operate effectively with a much smaller $\text{N}$ than MLPs. This is because if $m$ (spline degree) is small $(m= 2, 3,5)$, it can be seen as a constant. This not only reduces parameter count, but also enhances generalization and facilitates interpretability \cite{Ta2025PRKAN}. 

\subsubsection{Liner-time complexity of  b-spline functions by tensor forms}
\label{secc:LTCspline}
Therefore, in this section, we explain the idea for calculating the LTBs-KAN computational complexity. Consider that we are calculating the forward pass complexity, a list $[d_0, d_1, \dots, d_L]$, containing the input and output dimension of each KAN layer, can be defined, such that for layer $i$ we have the input dimension $d_i$ and the output one $d_{i+1}$
which specifies a sequence of $L$ layers in LTBs-KAN with mappings $d_0 \to d_1, d_1 \to d_2, \dots, d_{L-1} \to d_L$. Each layer in LTBs-KAN contains two sets of trainable parameters, a standard linear weight matrix, $\text{base weight} \in \mathbb{R}^{d_i \times d_{i-1}}$ and a spline-based weight tensor, $\text{spline weight} \in \mathbb{R}^{d_i \times d_{i-1} \times D}$, where $D = n + m$. Therefore, each layer contributes with $d_i  d_{i-1}  (1 + D)$
parameters, and the total number of trainable parameters in the network can be expressed as:
\begin{equation}
	P_{\text{total}} = \sum_{i=1}^{L} d_i  d_{i-1}  (1 + D).
\end{equation}
Thus, the computational complexity of the forward pass for a batch size $\mathcal{B}$, which, for a layer, includes the following components: $SiLU$ activation, $O(\mathcal{B} d_\text{in})$; residual addition, $O(\mathcal{B} d_\text{out})$; and the layer normalization, $O(\mathcal{B} d_\text{out})$. If $p, s, D$ are small positives constants, i.e.,  $m=2,3$, with $s=2,3$ and $p=4,5$, then each layer $i$ requires $T^{i}_{layer} = O\left( \mathcal{B} d_{i-1}  d_i \, (1 + D) \right) \approx O(\mathcal{B}  d_\text{in}  d_\text{out}  D)$,
operations due to the linear transformation and the spline-based tensor contraction. Therefore, the total forward-pass computational complexity is
\begin{equation}
	T_{\text{total}} = 
	O\left( 
	\mathcal{B} \sum_{i=1}^{L} d_{i-1}  d_i  (1 + D)
	\right).    
\end{equation}
Furthermore, based on the automatic differentiation theorem \cite{Baydin2018AutomaticDI}, the backward pass has the same computational complexity.

Now, we analyze the complexity of the layer components, and of a complete network built with them. Using Eq. \eqref{eq:W_weights}, the KAN layer updates the weight matrix by using the following expression $W \leftarrow W + a_{jk} M_{jk}$, which requires scaling and adding a $d_\text{in} \times d_\text{out}$ matrix.

On the other hand, the grid update operation involves computing the global minimum and maximum values, resulting in a computational cost of $O(\mathcal{B}d_\text{in})$ and the Knot recomputation has a complexity of $O(D)$. 
The LTBs Algorithm computes all fitted Bézier 
B-spline coefficients in a single pre-calculation step, requiring 
\begin{equation}\label{eq:complexbezier}
	O(nm + m^{2}) = O(nm),    
\end{equation}
time when $n \gg m$, and only $O(1)$ additional working memory beyond the coefficient tensor.
After this pre-computation, each B-spline evaluation is reduced to evaluating a polynomial of degree $m$, which can be performed in $O(m)$ time. In contrast, the Cox-De Boor formula (Theorem \ref{lem:BMCform}) incurs a computational cost of $O(nm^2)$. This approach replaces the quadratic dependence on the spline degree with a linear one, at the expense of an initial $O(nm)$ pre-computation. Consequently, the overall spline complexity per layer in the LTBs model is: $T_{\text{spline}} = O(\mathcal{B} d_\text{in}  d_\text{out}  D)$.

The next step involves copying coefficients into the spline weight, which has a computational cost of $O(d_\text{out} d_\text{in} D)$. Therefore, updating the grid requires 
\begin{equation}\label{eq:complTgridupdate}
	T_{\text{gridUpdate}}= O(\mathcal{B} d_\text{in} + n m + d_\text{out}  d_\text{in}  D),
\end{equation} 
operations. Notably, this operation is not performed during every forward pass. The total computational complexity of the LTBs-KAN model can be calculated as follows. For a network with widths $d_0, d_1, \dots, d_L$, the total complexity is:
\begin{equation*}
	T_{\text{network}}
	= \sum_{i=1}^{L} 
	O\left( d_{i-1}  d_i \, (p s + \mathcal{B}(1 + D)) \right).    
\end{equation*}

If all layers have a width of approximately value $d$, and we consider $L$ and $d$ constants, while  $m$, $p$, and $s$ are small positives constants, i,e. $m=2,3, 5$, with $p= 2,3$ and $s=4,5$. Then,
\begin{equation}
	T_{\text{LTBs-KAN}}
	\approx O\left( L d^2 (ps + \mathcal{B}(1 + D)) \right) \approx O(n).    
\end{equation}

Finally, the proposed factorization mainly enhances the model in terms of parameter efficiency rather than reducing the per-inference computational cost. In contrast, LTBs Algorithm \ref{algoritmo} and the alternative approach described in Eq. \eqref{eq:W_weights} provide improvements in both complexity and parameter usage compared to the traditional KAN formulation. In particular, they enable the grid-dependent spline computations to be shifted to a preprocessing stage with cost $O(n m)$, avoiding the repeated expense of $O(\mathcal{B} d_\text{in} n  m)$ operations during each forward pass.

\subsection{Parallel complexity analysis}

Next, we analyze the work $T_1$ and $T_\infty$ of the LTBs Algorithm under the standard Parallel Computation Model (PRAM) \cite{cook1986upper, fortune1978parallelism,blelloch1990PrefixSums}. The work $T_1$ corresponds to the
total number of operations executed sequentially, while the span $T_\infty$ represents the length of the critical path in the computation DAG (Directed Acyclic Graph), i.e.,  the time
required when the algorithm is executed with an unbounded number of processors. The ratio $T_1/T_\infty$ measures the theoretical parallelism of the algorithm. 
Several tasks in the LTBs-KAN layer require parallel operations such as computing sums, minima, maxima, dot products, or normalization
statistics that gets their best complexity by the use of parallel balanced reduction trees \cite{blelloch1990PrefixSums}. And although the implementation does not explicitly construct reduction trees,
the dominant tensor operations are executed by PyTorch on accelerated backends
such as CUDA, where calculations are queued asynchronously and delegated to
backend libraries optimized for linear algebra and BLAS operations
\cite{pytorch_cuda_semantics}.
Furthermore, PyTorch explicitly notes, in the context of associative scanning, that
associative computations can be parallelized using a tree reduction algorithm
rather than being executed sequentially \cite{pytorch_associative_scan}. Therefore,
in PRAM analysis, it is appropriate to model the underlying reductions
that appear in matrix multiplications, tensor contractions, and normalization statistics
using balanced reduction trees, resulting in a logarithmic complexity. Moreover, let us consider parallel operations over $\mathcal{M}\in \mathbb{N}$ elements, 
at each parallel step pairs of elements are combined,
halving the number of active values. After $k$ steps the number of remaining
elements becomes $\mathcal{M} / 2^k$. The parallel operations terminate when a single value remains,
which occurs when $\mathcal{M} / 2^k = 1$, yielding $k = \log_2 (\mathcal{M})$. Therefore, although the total work of parallel operations remains $O(\mathcal{\mathcal{M}})$, the span of the operation becomes $T_\infty = O(\log \mathcal{M})$.

Therefore, this logarithmic span appears repeatedly in the complexity analysis of the LTBs-KAN model, since procedures such as matrix multiplications, tensor contractions, and normalization layers involve parallel operations over
large vectors, see \cite{blelloch1990PrefixSums,cormen2009}.


As we saw in Section \ref{section:implementation_LBTs_KAN}, the LTBs Algorithm \ref{algoritmo} computes the tensor $\mathrm{C} \in \mathbb{R}^{n} \times (n+m) \times (m+1)$, where $n$ is the grid size and $m$ is the spline order. Thus, the sequential work required to compute all coefficients is $T_1 = O(n m + m^2)$, since typically $n \gg m$, the dominant term becomes (see Eq. \ref{eq:complexbezier}) in $T_1 = O(n m)$.  Therefore, the span is determined by the recurrence relations involved in the polynomial representation of the spline basis, which introduce dependencies proportional to the spline degree. Consequently, $T_\infty = O(m^2)$.
The resulting theoretical parallelism is therefore

{\footnotesize
\begin{equation}
	\frac{T_1}{T_\infty}
	=
	O\left(\frac{nm + m^2}{m^2}\right)
	=
	O\bigg(\frac{n}{m} + 1\bigg).    
\end{equation}}

\noindent When $n \gg m$, the LTBs Algorithm \ref{algoritmo} exhibits significant parallelism,
approximately $O(n/m)$, we can see more details about complexity in parallel algorithms in \cite{cormen2009}.

In this part, the grid update stage involves computing the minimum and maximum of the input tensor, generating a new knot grid, recomputing spline coefficients, and copying compatible tensor blocks. Thus computing the minimum and maximum of a tensor of size $\mathcal{B} d_\text{in}$ requires
parallel operations, and therefore has span $T_\infty = O(\log(\mathcal{B}  d_{\text{in}}))$, moreover, computing spline coefficients contributes an additional
$O(m^2)$ to the critical path. Therefore, the span of the gridUpdate procedure is

{\footnotesize
\begin{equation}
	T_\infty(\text{gridUpdate})
	=
	O(\log(\mathcal{B} d_\text{in}) + m^2).
\end{equation}}

The total work of this procedure is: $T_1(\text{gridUpdate}) =
O(\mathcal{B} d_\text{in} + nm + d_\text{out} d_\text{in} D)$,  Eq. \eqref{eq:complTgridupdate}. In the forward pass complexity, each LTBs-KAN layer performs a linear transformation, spline-based tensor
contraction, element-wise activation, and layer normalization. Next, the sequential work of the forward pass with batch size $\mathcal{B}$ is $T_1(\text{forward})=
	O(\mathcal{B}  d_\text{in} d_\text{out} D)$. In the forward execution element-wise operations such as SiLU activations have constant complexity not interfering in the complexity calculations for the span. Now, 
the construction of the factorized weight matrix involves operations
over $ps$ matrices, producing  a span of $O(\log(ps))$ matrix multiplications 
, while the spline
tensor contraction requires $d_\text{in} D$ operations and layer
normalization requires $d_\text{out}$ operations. Thus, the span of the forward pass is

\begin{equation}
	T_\infty(\text{forward})
	=
	O(\log(p s) + \log(d_\text{in} D) + \log(d_\text{out})).
\end{equation}

\noindent Finally, for a network with $L$ layers and widths
$d_0, d_1, \dots, d_L$, the total sequential work becomes

{\footnotesize
\begin{equation}
	T_1(\text{network})
	=
	O\left(
	\mathcal{B} \sum_{i=1}^{L} d_{i-1} d_i (1 + D)
	\right).
\end{equation}}

\noindent If all layers have approximately the same width $d$, this simplifies to $T_1 = O(L  \mathcal{B} d^2  D)$. Thus, 
the span only grows logarithmically with the layer dimensions, $T_\infty = O(\log(d\ D))$. In conclusion, the theoretical parallelism of the architecture is 

{\footnotesize
\begin{equation}
	\frac{T_1}{T_\infty}
	=
	O\left(
	\frac{L  \mathcal{B} d^2  D}{\log(d D)}
	\right).
\end{equation}}

\noindent This result indicates that the LTBs-KAN model admits substantial
parallelism and can efficiently exploit modern parallel hardware such as
GPUs and TPUs.


\section{KAN-ConvNet: the LTBs algorithm in convolutional neural networks}
\label{secc: KAN-ConvNet}

Convolutional Neural Networks (CNN's), first introduced by LeCun et al.~\cite{lecun2015}, are widely used in moderns architectures of deep learning due to their effectiveness in processing high-dimensional structured data, particularly image processing. These architectures typically employ convolutional layers consisting of linear transformations followed by nonlinear activation functions. 

However, in recent years, there has been increasing interest in incorporating advanced mathematical frameworks into deep learning architectures to enhance their expressive power and learning capabilities. Among these developments, as we mentioned earlier, we have KANs.  Motivated by these advances, in this section we investigate the extension of KAN-based architectures to convolutional layers, which are a fundamental component of CNN's widely used in computer vision applications. Conventional CNN's rely on fixed activation functions combined with linear convolutional operations. While effective, this formulation may limit the flexibility of the learned feature transformations. By incorporating spline-based convolutional operators, similar to the approach introduced in SplineCNN by Fey and Lenssen~\cite{fey2018splinecnn}, neural networks can achieve a richer representation of nonlinear relationships within structured data. In particular, we apply the LTBs-KAN model in  each spline-based convolutional operator.

\subsection{KAN-Conv2D: spline-based convolution}
\label{section:KAN-Conv2D}

First, we explain the KAN-Conv2D layer, this architecture  extends the concept of standard convolution by replacing the linear kernel mapping with a LTBs-KAN transformation. Instead of computing a dot product between a kernel and a local input patch, this layer applies a shared LTBs-KAN to each spatial patch. This preserves translational equivariance while increasing local expressiveness. Given an input tensor $\mathbf{x} \in \mathbb{R}^{\mathcal{B} \times C_{\text{in}} \times H \times W}$,  where $\mathcal{B}$ denotes the batch size, $C_\text{in}$ the number of input channels, and $(H, W)$ the spatial dimensions, i.e., Height  $(H)$ and Width $(W)$. The operation proceeds through next steps.

First, we have a Patch Extraction, then an \text{unfold} (im2col) operation, thus local tensor's windows are rearranged as $X_{\text{patch}} \in \mathbb{R}^{\mathcal{B} \times (C_{\text{in}}k_Hk_W) \times L}$, where $k_H$ and $k_W$ are the kernel dimensions, and $L = H_{\text{out}} W_{\text{out}}$ is the total number of output positions. Therefore, we apply a Patch Flattening, where  each local patch is reshaped into a vector of dimension $K = C_{\text{in}} k_H k_W$, we have then $P \in \mathbb{R}^{(\mathcal{B} L) \times K}$.  Therefore, we have spline transformation for each flattened patch as $Y = f_{\text{LTBs-KAN}}(P),$ where $f_{\text{LTBs-KAN}} : \mathbb{R}^{K} \rightarrow \mathbb{R}^{C_{\text{out}}}$ is composed of layers with spline activations i.e.,  B-splines that approximate nonlinear relationships between input patch elements and output responses. As a result, we have a spatial reassembly. The resulting outputs are reshaped to recover the spatial structure $Y \in \mathbb{R}^{\mathcal{B} \times C_\text{out} \times H_{\text{out}} \times W_{\text{out}}},$
with the output dimensions given by

{\footnotesize
\begin{equation*}
	H_{\text{out}} = \left\lfloor \frac{H + 2p_H - d_H(k_H-1) - 1}{s_H} + 1 \right\rfloor,
	W_{\text{out}} = \left\lfloor \frac{W + 2p_W - d_W(k_W-1) - 1}{s_W} + 1 \right\rfloor,
\end{equation*}}
where $(p_H,p_W)$, $(d_H,d_W)$  and $(s_H,s_W)$, denote padding, dilation, and stride for $H$ and $W$ respectively \cite{dumoulin2016conv}. Now, we compare with standard convolution $(*)$. Then, we have a conventional convolutional layer: $Y = \mathrm{\phi}(W * X + b)$, where $W$ is a linear kernel and $\phi$ is an activation function, in this work we use $\text{SiLU}$  (Eq. \ref{eq:basefunction}). In contrast, the KAN-Conv2D layer computes $Y = \mathrm{f_{\text{KAN}}}(X_{\text{patch}})$, where $f_{\text{KAN}}$ is a nonlinear spline operator applied identically across all spatial patches. Thus, we consider advantages like enhanced local expressiveness. Therefore, we have interpretability, where the spline formulation allows analysis of how each patch dimension contributes to the output. Finally, compatibility, i.e., the operator supports the same hyperparameters (stride, padding and dilation) as standard \text{Conv2d}.

Finally, the operator reduces exactly to a standard linear convolution implemented via \text{im2col + GEMM} if KAN has no hidden layers, thus constituting a direct generalization of the classical convolution.

\begin{figure}[htb!]
	\centering
	\includegraphics[width=.65\textwidth]{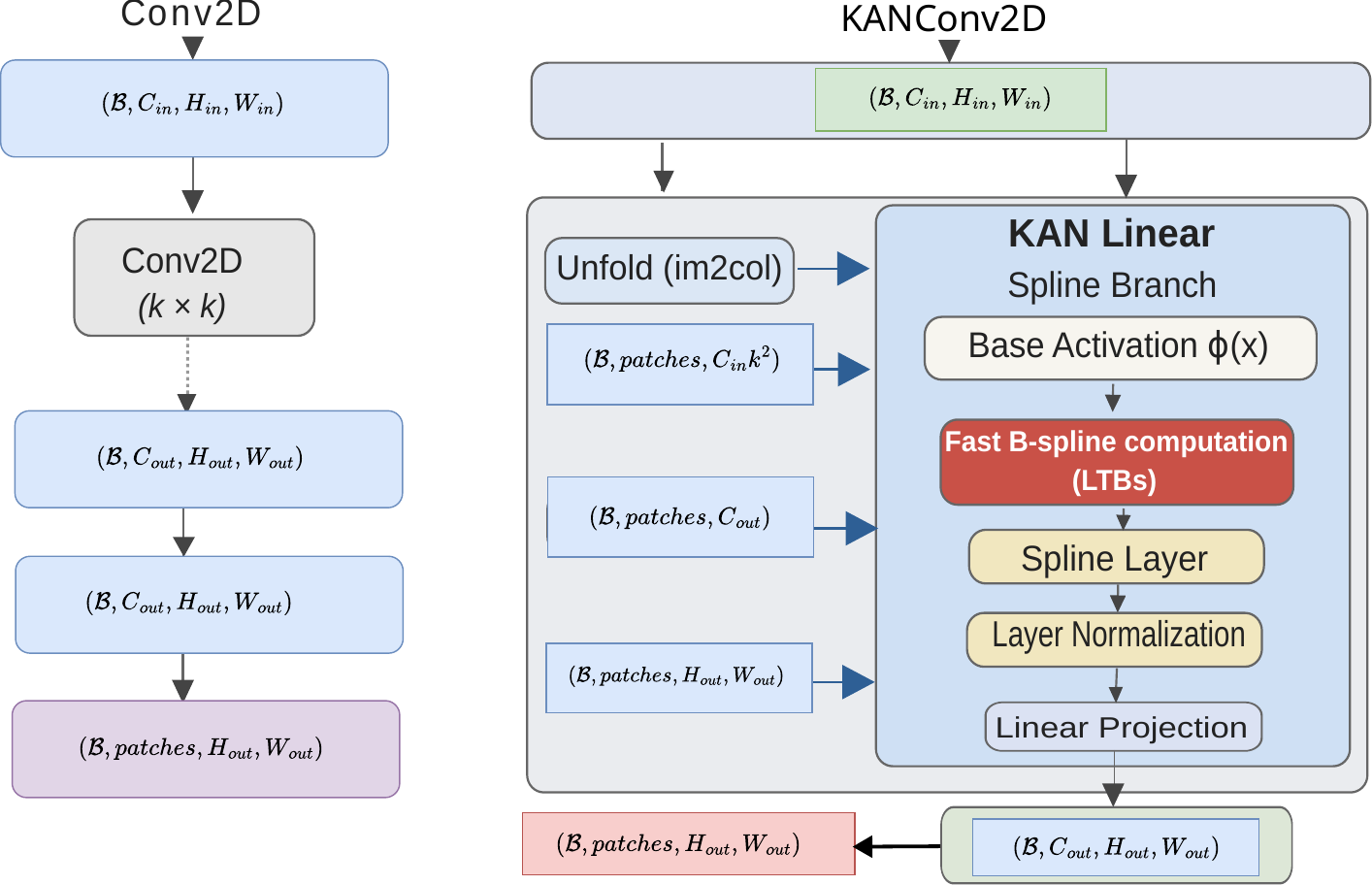}
	
	\caption{\label{fig:KANConvNet} Architecture of the Conv2D. Arquitecture Conv2D in LTBs-KAN,  generalized convolution replaces the filter-patch product with a shared KAN linear function. LTBs-KAN architecture is modified into a Convolutional Network for application to the CIFAR-10 dataset.}
\end{figure}

There are other implementations of the KAN Convolutional architecture \cite{LIU2026105983, LOU2026114405, WANG2026113042}. However, we find some issues with these implementations: For example, KAConvNet and KUNet  \cite{ LIU2026105983, LOU2026114405} embeds Kolmogorov-Arnold representations directly into Convolutional Kernels with linear and $3^{rd}$ degree splines. Instead, LTBs-KAN-ConvNet applies a nonlinear KAN-based mapping after patch extraction separating spatial aggregation from functional approximation. In another example, SpectralKAN \cite{WANG2026113042}, improves efficiency by reducing activation functions using Efficient-KAN. Instead, LTBs-KAN-ConvNet obtains efficiency through an optimized spline formulation and structured parameterization, avoiding redundancy and specialized pipelines. Given the previous descriptions, we take the decision of comparing Convolutional EfficientKAN against the LTBs-KAN-ConvNet to cover KUNet and SpectralKAN. Given that we are not able to find an implementation of KAConvNet, we are writing a new work to implement all those variations and compare them against LTBs-KAN-ConvNet in the future.

\section{Experiments}
\label{secc:experiments}

In this Section, we investigate the performance of the proposed model (LTBs-KAN), and for this, we conducted a series of experiments using MNIST, Fashion-MNIST  datasets \cite{LeCunn1998MNIST, Xiao2017FMNIST} and finally CIFAR-10 dataset \cite{Krizhevsky2009CIFAR10}. In the datasets MNIST and Fashion-MNIST the training procedure is conducted over multiple  epochs. During each epoch, a randomly chosen subset of training data is selected, referred to as a minibatch. In this case, each minibatch contains $64$ images. For MNIST and Fashion-MNIST datasets, we use a network architecture structure $[28\times 28, 32, 10]$, i.e., $784$ input neurons, $32$ hidden neurons, and $10$ output neurons. Finally, we experimented with the dataset CIFAR-10, this dataset consists of $60,000$, $32\times32$ colour images in $10$ classes, with $6,000$ images per class. The experiments with CIFAR-10 \cite{Krizhevsky2009CIFAR10} dataset, consisting of $10,000$ of samples of $32\times 32$ RGB images in $10$ classes, $6,000$ per class, we have the classic ratio of $90\%$ for training and $10\%$ test are chosen. Finally,  a $5$- fold cross-validation is applied to obtain the average results. 
Finally, using the Convolutional  architecture proposed in \cite{Krizhevsky2012ImageNet}, we generate a similar architecture by  substituting the Conv2D layers by the new proposed KAN-ConvNet, Fig. \ref{fig:KANConvNet}. The final network architecture structure in CIFAR-10 experiments can be seen in Fig. \ref{fig:CifarKANArchitecture}.

The way in which this separation between the training dataset and the test dataset has been carried out in these three experiments prevents the models from simply memorizing the training inputs and instead encourages generalization
to new and unknown data. We compare its performance with the results of the other models mentioned in Section \ref{Introduction}. In these experiments, we use $m=3$ as the spline order, the grid size $n=8$ and $\silu$ as activation function. In addition, we apply regularization techniques such as layer normalization \cite{ba2016layernormalization} and dropout \cite{Srivastava2014Dropout} with the probability of retaining a hidden unit being $0.1$ to avoid overfitting the models. 

\begin{figure}[H]
	\centering
	\includegraphics[width=0.75\textwidth]{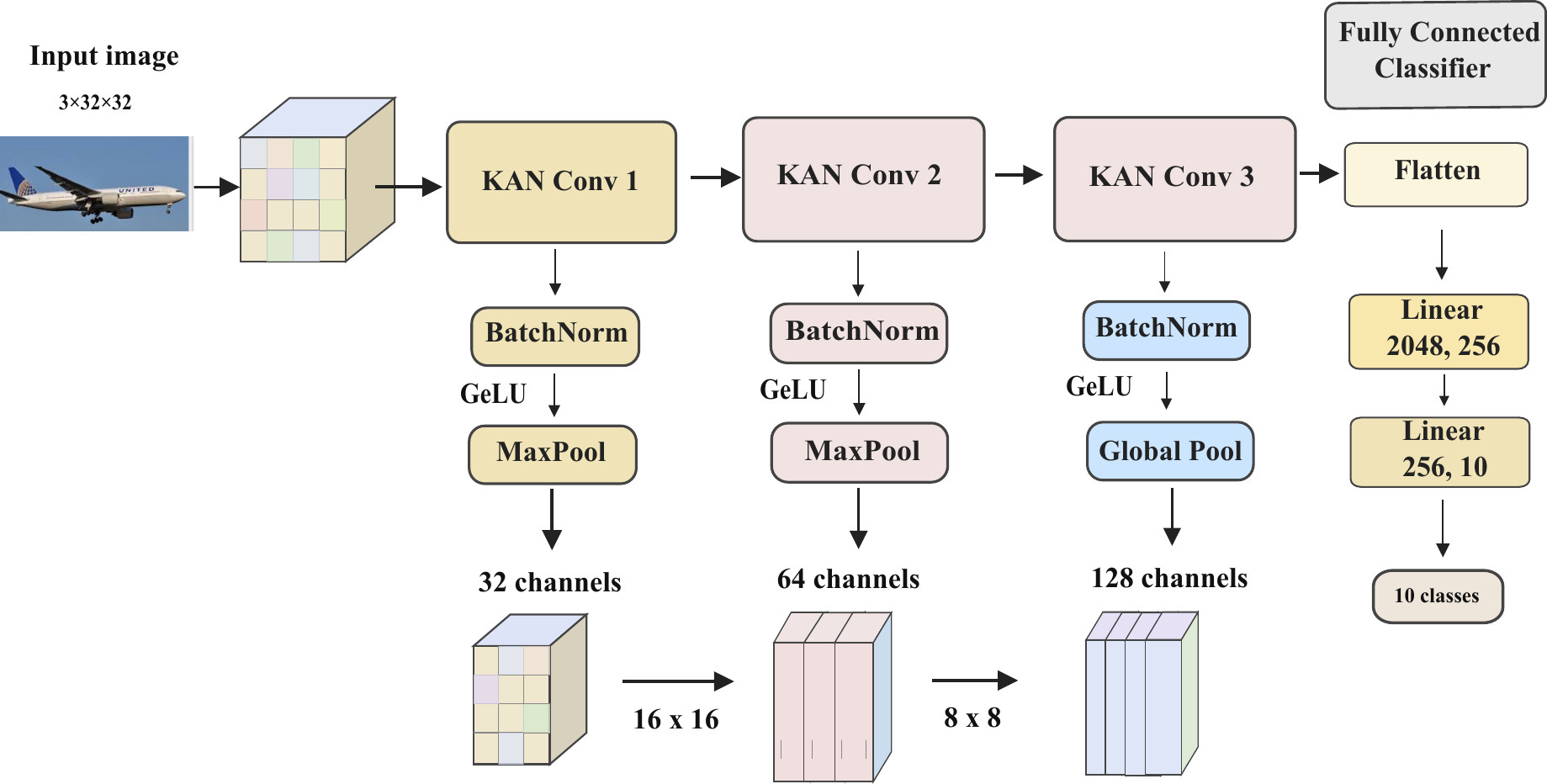}
	\caption{\label{fig:CifarKANArchitecture} LTBs-KAN-ConvNet architecture used in this experiment. The Max Pooling layers are done after every Convolutional Layer.}
\end{figure}

In addition, layer normalization  accelerates convergence during training. Furthermore, in this work we use the Kaiming initialization method, also known as Kaiming He initialization \cite{kaiming2015Rectifiers}, which can be used instead of Xavier initialization \cite{Glorot2010DifficultyTrainDeepNetworks}.
All the other hyperparameters are the same in all models, including batch size = $64$, learning rate = $\text{1e-3}$, weight decay =\text{1e-4}, the optimizer is AdamW \cite{loshchilov2017adamw} and the Adjusts the Learning Rate (LRScheduler  during optimization). 

Moreover, we employ CrossEntropy as loss fucntion in all models, this criterion computes the cross entropy loss between input logits and target.
Hyperparameter selection was guided by preliminary tuning experiments, with the objective of optimizing performance across the evaluated datasets, and  in the  sum-and-product decompositions of matrices for parameter reduction,  we use the values of $p=3$, $s=5$ for MNIST dataset and for Fashion-MNIST and CIFAR-10 datasets we apply $p=4$ and $s=8$.
Finally, the metrics used to evaluate the performance of the models in these experiments are: accuracy, precision, recall, and F1-score. The results have been obtained on a computer with $7$ 3700X processor with NVIDIA GeForce GT 710 graphics and $16$ GB of RAM (Used for MNIST and Fasion-MNIST datasets) and T4 GPU with no
more than 16 Gb memory available (used for CIFAR-10 dataset experiments).

\subsection{Experiments with MNIST}\label{secc:exp_MNIST}

We begin the experiments with MNIST dataset, we train all the models seen in Section \ref{Introduction} including our proposed model the LTBs-KAN.

\begin{table}[tb!]
\centering
\caption{
Comparison between LTBs-KAN and other KAN models over MNIST dataset after $15$ epochs. 
}\label{tab:MNIST}
\begin{threeparttable}
\footnotesize
\begin{tabular}{lcccccc}
\toprule
\textbf{Model} & \textbf{Params} & \textbf{Acc} & \textbf{Prec} & \textbf{Rec} & \textbf{F1} & \textbf{Time (s)} \\
\midrule
LTBs-KAN     & 203,378          & 0.9632          & 0.9631          & 0.9627          & 0.9628          & 244.27 \\
GottliebKAN  & \textbf{105,859} & 0.9581          & 0.9575          & 0.9578          & 0.9576          & 263.17 \\
BSRBF-KAN    & 306,528          & 0.9423          & 0.9437          & 0.9414          & 0.9419          & 517.56 \\
EfficientKAN & 254,080          & 0.9438          & 0.9489          & 0.9444          & 0.9448          & 366.65 \\
FastKAN      & 230,346          & 0.9619          & 0.9618          & 0.9613          & 0.9614          & 259.13 \\
MLP          & 643,538          & \textbf{0.9695} & \textbf{0.9692} & \textbf{0.9692} & \textbf{0.9692} & \textbf{194.66} \\
FasterKAN    & 204,896          & 0.9602          & 0.9599          & 0.9595          & 0.9597          & 237.57 \\
\bottomrule
\end{tabular}
\end{threeparttable}
\end{table}

\begin{figure}[tb!]
	\centering
	\includegraphics[width=0.55\textwidth]{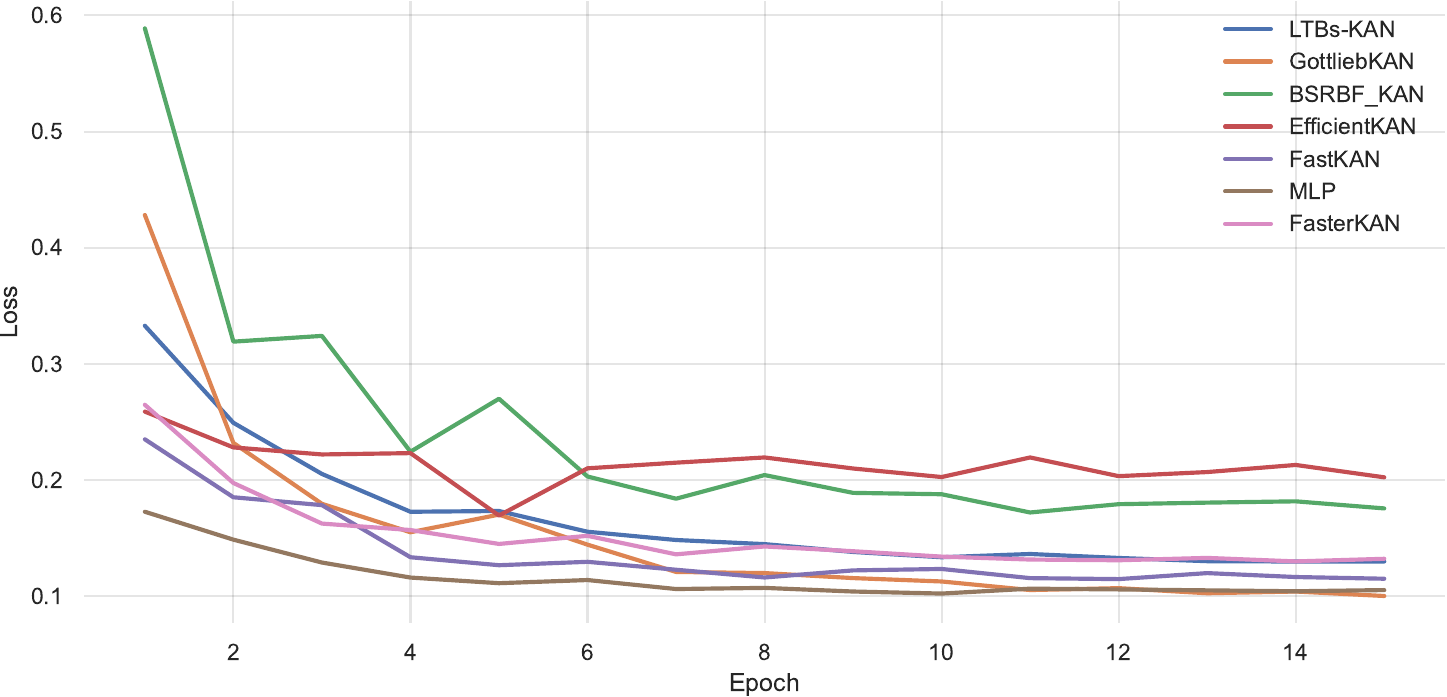}
	\caption{Test loss values during a training run of the different models after 15 epochs over MNIST dataset. The LTBs-KAN, outperforms EfficientKAN,  Gottlieb-KAN, BSRBF-KAN and FasterKAN.}
	\label{fig:MNIST_compartive}
\end{figure}

LTBs-KAN shows good results after 15 epochs of training, at testing, as it can be seen in Table \ref{tab:MNIST} against most popular KAN based architectures. However, it is necessary to point out that the MLP still has the best results with 643,538 parameters. Nevertheless, the LTBs-KAN architecture maintains good performance using the parameter reduction method described in Section \ref{secc:parametersmodel}.  

Although the LTBs-KAN architecture does not achieve the best loss reduction, Fig. \ref{fig:MNIST_compartive}, in addition of not having the best accuracy, precision, and recall when compared to MLP, the differences are quite minimal as for example for the precision metric is in the range of 0.0061. However, compared in training speed with respect to all the KAN architecture only FasterKAN is faster by seven seconds. Clearly, MLP is still faster given that it is not dealing with recurrence to calculate the splines.

Finally, we need to point out that even though the MLP gets the best performance, the LTBs-KAN architecture has only 203,378 parameters compared with the full MLP with 643,538 parameters. A model with less parameters is the GottliebKAN architecture, but LTBs-KAN architecture has better metrics due to the fact that splines have higher expressibility given their piecewise flexibility \cite{Boor1978Splines}. This makes the LTBs-KAN suitable for environments with scarce resources keeping with the expressibility capabilities of the KAN architecture \cite{liu2024kan}. Therefore, it is clear that the proposed LTBs-KAN model has the best scores for the KAN architectures by balancing speed and the final number of parameters. It is more, the training time for the MNIST dataset, confirms the computational complexity of the proposed parallel algorithm, Section \ref{linearKAN:complexity}. 

Therefore, the model demonstrates good performance  on the MNIST dataset, achieving both precision and recall, due to the methods described in the previous sections. Finally, in Section \ref{secc:experimentsCIFAR10}, it is possible to see a confirmation on  the utility of the new LTBs-KAN model for real life applications. 

\subsection{Experiments with Fashion-MNIST} \label{secc:exp_FashionMNIST}

In the second part of the experiments, we trained LTBs-KAN with the Fashion-MNIST dataset over 20 epochs before testing. All models are trained under the same hyperparameter configuration with results presented in Table \ref{tab:Fashion_MNIST}. In addition, the LTBs-KAN architecture  is trained with reduction parameters of $p=4$ and $s=8$. Here, we can see that the LTBs-KAN is not the fastest model during training, see Column Time, being surpassed by FastKAN, FasterKAN, GottliebKAN and MLP. Here, it is possible to see that BSRBF-KAN, FastKAN, FasterKAN and MLP obtain better metrics.

\begin{table}[!tb]
\centering
\caption{
The comparison between LTBs-KAN and other KANs, with best metric values in training runs over the Fashion-MNIST dataset after 20 epochs.
}\label{tab:Fashion_MNIST}
\begin{threeparttable}
\footnotesize
\begin{tabular}{lcccccc}
\toprule
\textbf{Model} & \textbf{Params} & \textbf{Acc} & \textbf{Prec} & \textbf{Rec} & \textbf{F1} & \textbf{Time (s)} \\
\midrule
LTBs-KAN     & 203,412          & 0.8744          & 0.8735          & 0.8744          & 0.8736          & 378.84 \\
GottliebKAN  & \textbf{105,859} & 0.8671          & 0.8659          & 0.8671          & 0.8658          & 350.43 \\
BSRBF-KAN    & 306,528          & 0.8747          & 0.8740          & 0.8747          & 0.8738          & 693.55 \\
EfficientKAN & 254,080          & 0.8536          & 0.8575          & 0.8536          & 0.8532          & 486.40 \\
FastKAN      & 230,346          & \textbf{0.8850} & \textbf{0.8843} & \textbf{0.8850} & \textbf{0.8843} & 327.27 \\
MLP          & 643,538          & 0.8773          & 0.8765          & 0.8773          & 0.8766          & 256.42 \\
FasterKAN    & 204,896          & 0.8831          & 0.8827          & 0.8831          & 0.8828          & 313.64 \\
\bottomrule
\end{tabular}

\end{threeparttable}
\end{table}

\begin{figure}[tb!]
	\centering
	\includegraphics[width=0.55\textwidth]{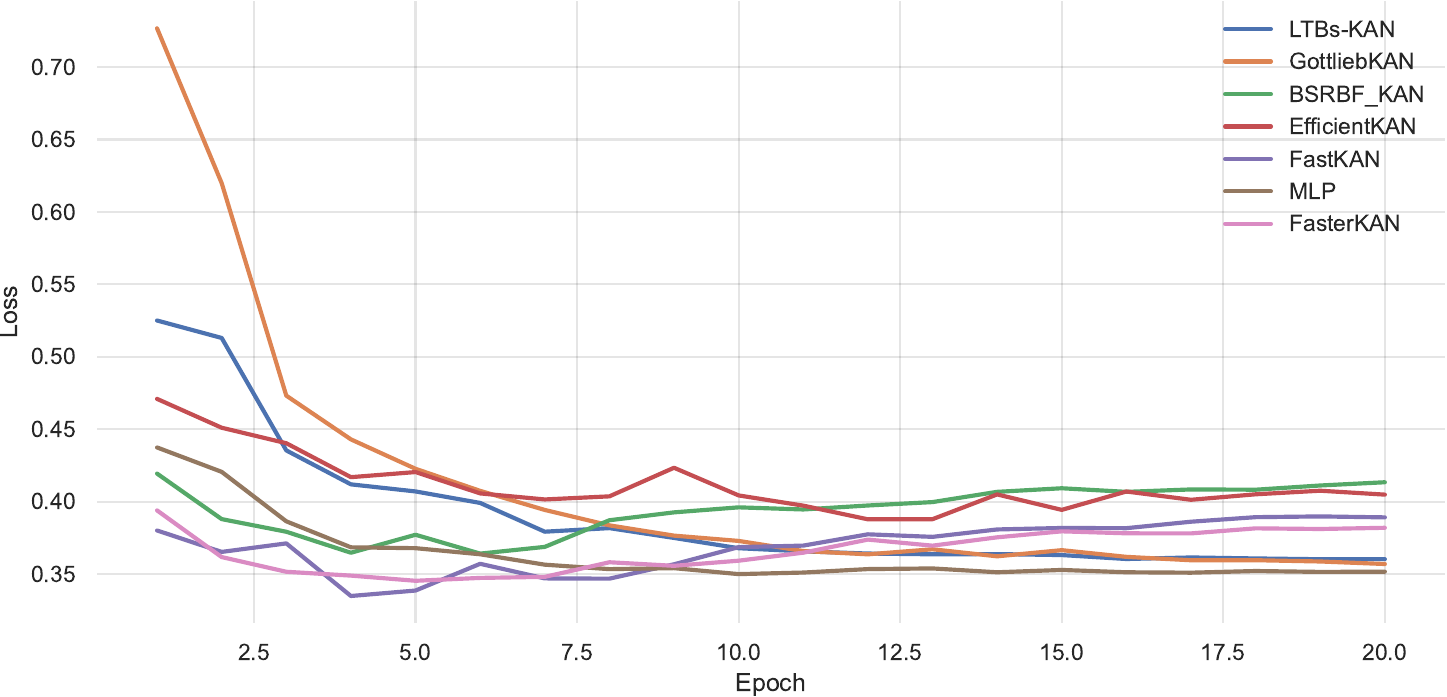}
	\label{fig:FMNIST}
	\caption{Test loss values during a training run of the different models over $20$ epochs using the Fashion-MNIST dataset. We see that the LTBs-KAN has the nearest loss convergence to the MLP convergence at epoch 20. }
\end{figure}

One way to explain these results is the basic fact that BSRBF-KAN, FastKAN, FasterKAN and MLP have more parameters than the LTBs-KAN architecture. In addition, FastKAN,
and FasterKAN
use a linear function as spline which allows to reduce complexity during training improving the learnability of the information in the data set. This is a phenomenon that happens in learning known as the "bias-variance dilemma" \cite{hastie2009elements} which points out to the need to increase the size of the dataset when using more complex functions for learning. This hypothesis surges from the fact that 
BSRBF-KAN and LTBs-KAN have almost the same performance in the metrics (with a small difference of 0.0003) using a spline of degree three which is more complex than a simple linear function at the other architectures including MLP.

This points out to the fact that it is necessary to introduce a regularization into the degree of the spline to control de complexity of the spline when is necessary to have a less complex learning model. This is a future work in the research of the LTBs-KAN architecture.

\subsection{Experiments with CIFAR-10}\label{secc:experimentsCIFAR10}

For the experiment with the CIFAR-10 dataset, three different architectures are proposed: the fully connected MLP, a simple AlexNet using convolutional layers and the LTBs-KAN ConvNet architecture. For this last architecture, we used the proposed KAN-Conv2D layer which is defined in Section \ref{section:KAN-Conv2D}. 

The test losses for the three models can be seen in Fig. \ref{fig:test_loss_convergence_cifar10} during 20 training epochs. It is clear that LTBs-KAN ConvNet converges as fast as the EfficientKAN and faster than the other two architectures during the training, reaching their best loss in epoch 8.

\begin{table}[tb!]
\centering
\caption{
Comparison of three models over the CIFAR-10 dataset in time training for $20$ epochs with training time in seconds, after applying dropout and layer normalization.
}
\label{table:CIFAR-10}

\begin{threeparttable}

\footnotesize
\begin{tabular}{lcccccc}
\toprule
\textbf{Model} & \textbf{Params} & \textbf{Acc} & \textbf{Prec} & \textbf{Rec} & \textbf{F1} & \textbf{Time} \\
\midrule
LTBs-KAN ConvNet* & 1,538,954 & 0.8179 & 0.8173 & 0.8179 & 0.8174 & 956.83 \\
LTBs-KAN ConvNet** & 2,811,210 & 0.8254 & 0.824123 & 0.8254 & 0.824431 & 1734.60 \\
Simple AlexNet & 2,474,506 & 0.8208 & 0.820917 & 0.8208 & 0.820806 & 419.10 \\
MLP & 9,656,323 & 0.5144 & 0.510917 & 0.5144 & 0.512068 & 341.12 \\
EfficientKAN-ConvNet & 1,457,802 & 0.8222 & 0.822073 & 0.8222 & 0.821975 & 3298.76\\
\bottomrule
\end{tabular}
\end{threeparttable}
\end{table}
\begin{figure}[tb!]
	\centering
	\includegraphics[width=0.55\textwidth]{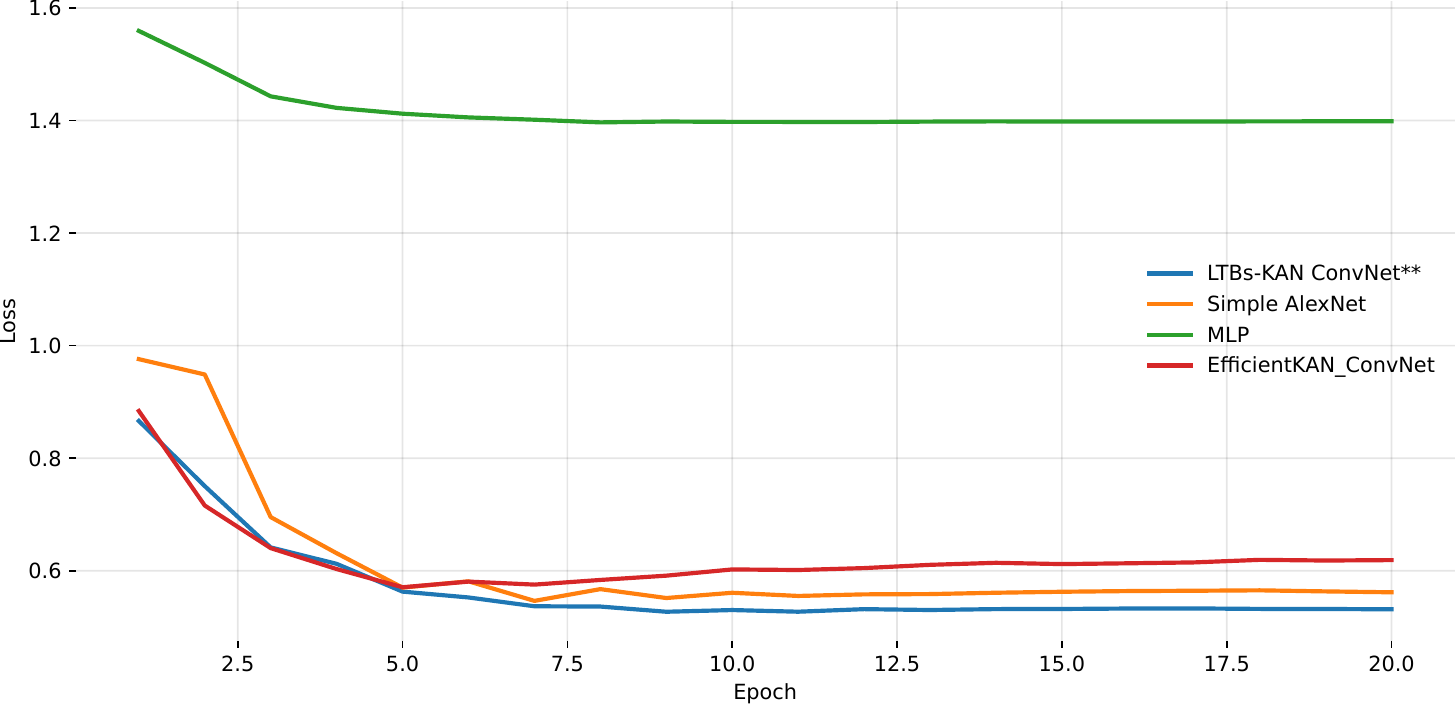}
\caption{\label{fig:test_loss_convergence_cifar10} Loss during a training run of 20 epochs over the CIFAR-10 dataset for MLP, AlexNet, EfficientKAN-ConvNet and LTBs-KAN ConvNet. LTBs-KAN ConvNet* uses the parameters $p = 4$ and $s = 8$ and LTBs-KAN ConvNet** uses $p= 12$ and $s =12$}
\end{figure}

Our results on CIFAR-10 are summarized in Table \ref{table:CIFAR-10}. The LTBs-KAN ConvNet** architecture outperforms the other two models in all the metrics of accuracy, precision, recall and F1. 

Although EfficientKAN-ConvNet is the architecture with the fewest number of parameters (1,457,802) compared to  LTBs-KAN ConvNets, it is possible to see that against LTBs-KAN ConvNet* there is a difference of 0.0043, for example, at accuracy with 3.4 times faster execution time.

Finally, although the time complexity of LTBs-KAN ConvNet is three times higher than the one in fully connected MLP and two times higher than that of the AlexNet architecture, it is clear that the expressibility of the KAN helps the LTBs-KAN ConvNet extract more information from the dataset to obtain better performance than the two other architectures.

\section{Conclusions}\label{secc:conclusions}

We introduce LTBs-KAN model to compute the basis of the B-splines with time complexity  $O\left( \log\left( dD\right) \right)$. This efficiency comes from tree complementary mechanisms: First, the use of tensor parallelism computation. Second, the dense linear transformation is replaced by a factorized sum of fixed matrices with only a small set of trainable coefficients. Third, the use of  adaptive knot and grid updates to improve structural consistency and time efficiency. Using LTBs algorithm, we define a new KAN layer for architectures to test against three data sets: MNIST, Fashion-MNIST and CIFAR-10. 

In the first data set, MNIST, the performance of the LTBs-KAN is the best when compared to the other KAN architectures with the second best number of parameters only surpassed by the MLP. In the second data set, Fashion-MNIST, BSRBF-KAN, FastKAN and FasterKAN surpass the proposed LTBs-KAN architecture. In these, the performance difference is a factor of the complexity of the data vs the complexity of the LTBs-KAN. Then, raising the need to develop a regularization on the spline degree for model complexity.

At the CIFAR-10 data set, the proposed KAN-Conv2D layers allowed the LTBs-KAN ConvNet architecture to have the best scores in the CIFAR-10 data set. Not only that, but the LTBs-KAN ConvNet is smaller, only EfficientKAN-ConvNet is smaller,of the four architectures tested at CIFAR-10 with respect to the number of parameters. 

Finally, in future work, we will continue to explore the combination of KAN and their components, as regularization and sparsity promotion, for the LBTs-KAN model. 


\end{document}